\documentclass[a4j,tikz]{article}
\usepackage{jsaisig2}
\usepackage{graphicx}

\usepackage{amsmath, amsfonts, bbm, amsthm}
\usepackage{booktabs}
\usepackage{comment}%複数行コメントアウトを使用可能にする
\usepackage{algorithm}
\usepackage{algpseudocode}

\usepackage{tikz}
\usetikzlibrary{backgrounds} % 辺を頂点の背面に描くため
\usetikzlibrary{arrows.meta} % 矢印の形状指定
\usetikzlibrary{bending}

\usepackage{enumitem}
\usepackage{multirow}
\newtheorem{theorem}{Theorem}[section]
\newtheorem{lemma}[theorem]{Lemma}  % 定理と番号を共有

\newtheorem{corollary}[theorem]{Corollary}

\newlength{\ublift}
\newcommand{\ub}[2]{\raisebox{-\ublift}{$\displaystyle\underbrace{#1}_{#2}$}}

\newcommand{\argmax}{\mathop{\rm arg~max}\limits}
\newcommand{\argmin}{\mathop{\rm arg~min}\limits}
\newcommand{\E}{\mathbb{E}}%期待値
\newcommand{\p}{{\boldsymbol{p}}}
\newcommand{\D}{\overline{D}}%楽観cdf分布
\newcommand{\F}{\mathcal{F}}%事象
\renewcommand{\a}{\alpha}
\renewcommand{\b}{\beta}

\newcommand{\G}{\mathcal{G}}
\newcommand{\bad}{\text{bad}}

\allowdisplaybreaks[4]

\begin{document}

% 和文タイトル
%\title{\Large \textsf{確率的セミバンディット設定における最小最長パスk経路選択}}
\title{}
% 英文タイトル
\etitle{Selecting k Paths with the Minimum Longest Path Length in the Stochastic Semi-Bandit Setting} 

% 著者名: 
%	・各著者を\quad（全角空白）区切りで列挙
% 	・著者名の直後に\afil{所属番号}を追加→所属番号を上付で出力（\textsuperscript{所属番号}と同じ）
% 	　複数機関へ所属している場合は番号をカンマ区切りで列挙（下記著者２参照）
% 　・Corresponding Authorについては所属の後に\thanksを続け，連絡先を記入
%	・英文著者はカンマ区切りで列挙
\author{Shunsuke Aoki\afil{1}%
  \thanks{ Kita 14, Nishi 9, Kita-ku, Sapporo, Hokkaido, 060-0814, Japan \newline%
		E-mail:aoki.shunsuke.a1@elms.hokudai.ac.jp }\quad%
\quad%
	Atsuyoshi Nakamura\afil{1}\thanks{E-mail:atsu@ist.hokudai.ac.jp}%
%	\thanks{連絡先：北海道大学大学院情報科学研究院 \newline%
%	　　　　　　〒060-0814札幌市北区北14条西9丁目 \newline%
%		　　　　　　E-mail:atsu@ist.hokudai.ac.jp  }\\%
}

% 所属
\affiliation{%
	\afil{1} Graduate School of Information Science and Technology, Hokkaido University \\
%	\afil{2} 北海道大学大学院情報科学研究院\\
%	\afil{2} Graduate Faculty of Information Science and Technology, Hokkaido University
}
\if0
\jauthor{青木 俊輔\afil{1}%
	\thanks{連絡先：〒060-0814札幌市北区北14条西9丁目 \newline%
		　　　　　　E-mail:aoki.shunsuke.a1@elms.hokudai.ac.jp }\quad%
	中村 篤祥\afil{1}\thanks{連絡先：E-mail:atsu@ist.hokudai.ac.jp}\\%
%	\thanks{連絡先：北海道大学大学院情報科学研究院 \newline%
%	　　　　　　〒060-0814札幌市北区北14条西9丁目 \newline%
%		　　　　　　E-mail:atsu@ist.hokudai.ac.jp  }\\%
}
\fi
\jauthor{}
% 所属
\jaffiliation{%
	%\afil{1} 北海道大学大学院情報科学院 \\ 
%	\afil{1} Graduate School of Information Science and Technology, Hokkaido University \\
%	\afil{2} 北海道大学大学院情報科学研究院\\
%	\afil{2} Graduate Faculty of Information Science and Technology, Hokkaido University
}

\abstract{
When performing %simultaneous communication over
parallel data transmission through
a network using multiple paths, it is practically important to minimize the maximum %communication 
transmission time among the selected paths. This study addresses an online problem in which %a specific number of
$k$ paths from an origin vertex to a destination vertex must be selected at each time step within a network represented as a directed graph. Here, %each network edge has a defined capacity for the number of paths it can accommodate,
 the number of paths going through each edge in each parallel data transmission is limited to its capacity,
and the time required for %traversal 
transmission is determined stochastically. We formulate the semi-bandit problem of selecting a set of paths to minimize the maximum traversal time among the selected paths and propose an algorithm to solve it.
}

% Revised abstract wording for clarity
% (minor edits above; retained original technical content)

\maketitle
\thispagestyle{empty}

%%%%%%%%%%%%%%%%%%%%%%%%%%%%%%%%%%%%%%

%\section{はじめに}
\section{Introduction}

  Problems of controlling flows on networks arise in many real-world applications.
  For instance, vehicle flows in road networks may be controlled to relieve congestion, and data flows in communication networks may be controlled to improve transfer efficiency.
  Flow control on these networks can be achieved by choosing which paths traffic should follow.
  The well-known shortest-path problem asks for a path that minimizes the total traversal time from an origin to a destination.
  
  In this paper we consider an online problem in which a player selects $k$ paths from an origin vertex to a destination vertex in a given directed graph at each time step.
  We assume that the traversal time of a path equals the sum of the traversal times of its edges, and that each edge's traversal time is a random variable drawn independently from an unknown distribution.
  The player does not know the expected traversal times in advance and only observes samples by actually traversing edges. At each time the player must choose the $k$ paths based on the observations collected so far.
%We consider the problem of selecting $k$ paths from a source vertex to a destination vertex in a given graph.

  The player's objective at each time is to minimize the maximum traversal time among the selected paths, and the goal over a horizon is to minimize the cumulative regret relative to the best fixed choice of $k$ paths.
  Minimizing the total traversal time of the selected paths does not necessarily minimize the longest traversal time among them.
\begin{figure}[bt]
  \centering 
  \begin{minipage}[c]{0.46\columnwidth}
  \begin{tikzpicture}
    \node[fill=black, shape=circle, scale=0.5] (flow) at (0,0) {};
    \node[fill=black, shape=circle, scale=0.5] (al) at (1,1) {};
    \node[fill=black, shape=circle, scale=0.5] (sl) at (1,-1) {};
    \node[fill=black, shape=circle, scale=0.5] (ar) at (3,1) {};
    \node[fill=black, shape=circle, scale=0.5] (sr) at (3,-1) {};
    \node[fill=black, shape=circle, scale=0.5] (sink) at (4,0) {};
    \draw (flow) edge[draw=red,->] node[font=\tiny\ttfamily,above] {2/1} (al);
    \draw (flow) edge[draw=red,->] node[font=\tiny\ttfamily,above] {2/1} (sl);
    \draw (al) edge[draw=red, ->] node[font=\tiny\ttfamily,above] {1/1} (ar);
    \draw (sl) edge[draw=red, ->] node[font=\tiny\ttfamily,above] {9/1} (sr);
    \draw (ar) edge[draw=red, ->] node[font=\tiny\ttfamily,above] {2/1} (sink);
    \draw (sr) edge[draw=red, ->] node[font=\tiny\ttfamily,above] {2/1} (sink);
    \draw (al) edge[->] node[font=\tiny\ttfamily,above] {6/1} (sr);
    \draw (sl) edge[->] node[font=\tiny\ttfamily,below] {6/1} (ar);
  
  \end{tikzpicture}
  \end{minipage}
  \ \ \ \ 
   \begin{minipage}[c]{0.46\columnwidth}
  \begin{tikzpicture}
    \node[fill=black, shape=circle, scale=0.5] (flow) at (0,0) {};
    \node[fill=black, shape=circle, scale=0.5] (al) at (1,1) {};
    \node[fill=black, shape=circle, scale=0.5] (sl) at (1,-1) {};
    \node[fill=black, shape=circle, scale=0.5] (ar) at (3,1) {};
    \node[fill=black, shape=circle, scale=0.5] (sr) at (3,-1) {};
    \node[fill=black, shape=circle, scale=0.5] (sink) at (4,0) {};
    \draw (flow) edge[draw=red,->] node[font=\tiny\ttfamily,above] {2/1} (al);
    \draw (flow) edge[draw=red,->] node[font=\tiny\ttfamily,above] {2/1} (sl);
    \draw (al) edge[->] node[font=\tiny\ttfamily,above] {1/1} (ar);
    \draw (sl) edge[->] node[font=\tiny\ttfamily,above] {9/1} (sr);
    \draw (ar) edge[draw=red, ->] node[font=\tiny\ttfamily,above] {2/1} (sink);
    \draw (sr) edge[draw=red, ->] node[font=\tiny\ttfamily,above] {2/1} (sink);
    \draw (al) edge[draw=red, ->] node[font=\tiny\ttfamily,above] {6/1} (sr);
    \draw (sl) edge[draw=red, ->] node[font=\tiny\ttfamily,below] {6/1} (ar);   
 
  \end{tikzpicture}
  \end{minipage}
  \ \\
  \ \\
  \begin{minipage}[c]{\columnwidth}
    \begin{center}
    \begin{tabular}{l|c|c}
    selected paths (red edges) & left fig. & right fig.\\
    \hline
    Summation& 18 & 20\\
    Longest Path& 13 & 10\\
    \hline
    \end{tabular}
    \end{center}
  \end{minipage}
  \label{fig:1-1}
  \caption{Candidates for the selection of two paths; each edge is labeled with the expected traversal time / capacity (the number of paths that can pass through).}
 \end{figure}
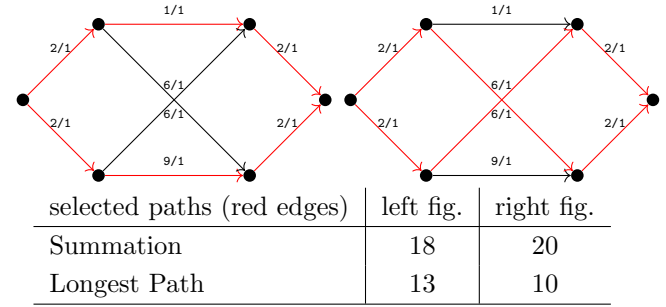
Figure \ref{fig:1-1} shows two ways of choosing 2 paths for a certain network.
%  If the selection shown on the left of Figure~\ref{fig:1-1} is made, the sum of the expected traversal times of the constituent edges is minimized, and in this respect it can be called an optimal solution.
%  On the other hand, if we consider the path with the largest expected traversal time, it is better to make the selection shown on the right of Figure~\ref{fig:1-1}. In this case, the sum of the expected traversal times is not minimized, but the traversal time of the path with the longest expected traversal time is minimized.
The selection on the left of Figure~\ref{fig:1-1} minimizes the sum of the expected traversal times of the constituent edges while the selection shown on the right of Figure~\ref{fig:1-1} minimizes the traversal time of the path with the longest expected traversal time.
In real-world logistics and network control, there are cases %that
in which methods are evaluated in terms of the maximum delay.
%For that purpose, even if the overall metric becomes somewhat worse, it is important to minimize the worst case.
In such case, it is important to minimize the worst component metric even if the overall metric becomes somewhat worse.

  In this study we propose SDCB-MMLP (SDCB for MinMax-Length Path), an algorithm that uses the SDCB (Stochastically Dominant Confidence Bound) policy \cite{generalreward-Chen}.
  We demonstrate the effectiveness of the proposed method by comparing it to a CUCB-based baseline \cite{pmlr-v38-kveton15} and by presenting a problem instance where the CUCB-based method fails to converge to the optimal solution.

\section{Related Work}
   This study can be regarded as a kind of combinatorial bandit problem studied by Gai et al. \cite{6166915}.
   In \cite{6166915}, the flow on all edges is represented by a single vector and
   actions are selected from a set $\mathcal{F}$ of finitely many vectors that is fixed in advance.
  In our work, we assume that each edge $i$ has a capacity $u_i\in \mathbb{N}$  $(u_i>0)$,
  and our problem is obtained by setting
  $\mathcal{F}=\{\boldsymbol{a}\in \prod_{i=1}^{|E|}\{0,1,\dots,u_i\} | \boldsymbol{a}\text{ satisfies the flow constraints}\}$,
   %   Here,
   where $E$ is the edge set of the given graph.
   %In \cite{6166915}, the inner product of an action vector $\boldsymbol{a}$ selected from the decision set $\mathcal{F}$ and the reward vector $\boldsymbol{x}$ generated from a distribution is the player's reward,
  In addition to the allowable action set $\mathcal{F}$, their reward setting is different from ours: the player's reward in \cite{6166915} is the inner product of an action vector $\boldsymbol{a}$ and a random reward vector $\boldsymbol{x}$, and the problem of maximizing cumulative expected reward is studied.
 .
  % As a setting closer to our problem, the problem of selecting a combination of arms whose bottleneck is the best has also been considered.
   %  In \cite{pure-ex}, a pure exploration problem is addressed, in which a path between two vertices on a graph is selected so that the reward of the bottleneck edge is maximized.
%  In our problem, the bottleneck path is related to the objective function, and the problem is one of selecting a set of paths.
   As a study closer to our reward setting, Du et al. \cite{pure-ex} consider the reward of the bottleneck component in a pure-exploration problem of selecting a combination of arms: their problem is to identify a path between two vertices that maximizes the minimum edge reward.
  Our problem can be seen as dealing with the reward of a bottleneck path.

\section{Problem Formulation}
  In a directed graph $G=(V,E)$,
  let the sets of vertices $V$ be $V=\{1,\dots,|V|\}$ and edges $E$ be $E=\{1,\dots,|E|\}$, respectively, where $|S|$ denotes the number of elements of a set $S$.
  We call $1\in V$ the source of this directed graph $G$, and $|V|\in V$ the sink of $G$.
  In addition, the start vertex and the end vertex of an edge $i \in E$ are denoted by $b(i) \in V$ and $e(i) \in V$, respectively.
  Let the capacity of each edge of this graph be $u_i \in \mathbb{N}$.
  The traversal time $X_i\in[0,1]$ of edge $i$ is assumed to be drawn independently every time the edge is traversed, according to a probability distribution $D_i$. Furthermore, the traversal times $X_i$ of the edges $i$ are assumed to be mutually independent. The list of traversal time distributions corresponding to the list of edges $(1,\dots,|E|)$ is denoted by $\boldsymbol{D}=(D_1,\dots,D_{|E|})$.
  When a list of edges $p=(i_1,i_2,\dots,i_{\ell})$ satisfies $s=b(i_1)$, $e(i_1)=b(i_{2})$,$\dots$, $e(i_{\ell-1})=b(i_{\ell})$, $e(i_{\ell})=t$, $p$ is a path from vertex $s$ to vertex $t$. A list of edges is not a set, but for simplicity of notation, whether or not an edge $i$ is an element of the list $p$ is denoted by $i\in p$ and $i\not\in p$, and $|p|=|(i_1,\dots,i_{\ell})|$ denotes the length $\ell-1$.
  At each time $n=1,2,\dots$, the player performs the following procedure.
  \begin{enumerate}
    \item Selects the set of $k$ paths $\boldsymbol{p}(n)=\{p_{1}(n), p_{2}(n), \dots , p_{k}(n) \}$ from the vertex $1$ to the vertex $|V|$ . 
  \item Observes the traversal time $X_{i,j} (n) \in [0,1]$ for edge $i$ of each path $p_{j}(n)$. 
    \item Suffer the maximum traversal time 
    $\displaystyle
      L(n) = \max_{1\leq j \leq k} L_j (n) 
    $
    of each path's traversal time
    $\displaystyle
      L_j(n) = \sum_{i\in p_{j}(n)} X_{i,j}(n)
    $
    .
  \end{enumerate}  
  Let $a_i(\boldsymbol{p})$ be the number of times edge $i$ is selected in the set of $k$ paths $\boldsymbol{p}=\{p_1,p_2,\dots p_k\}$.
  Then, $a_i(\boldsymbol{p})$ can be expressed as a function of the set of paths $\boldsymbol{p}$ using the following equation:
  \begin{align*}
    a_i(\boldsymbol{p}) = \sum_{j=1}^{k} \mathbbm{1} \{ i \in p_{j} \}.
  \end{align*}
  Here, $\mathbbm{1}\{A\}$ is the indicator function, which returns one when predicate $A$ is true and returns zero when $A$ is false. 
  For each edge $i\in E$, the set of paths at each time must be selected so that the capacity constraint
  \[
    a_i(\boldsymbol{p})\le u_i
  \]
  is satisfied.
  Let $\F$ be the set of $k$ paths from the vertex $1$ to the vertex $|V|$ that contain no cycle and satisfy the capacity constraint.

  The traversal time $X_{i,j}(n)$ of the edge $i$ at time $n$ shared by several paths $p_j(n)$ are assumed to be generated independently. %, even at the same time.
  The goal of the player is to minimize the traversal time of the path with the longest expected traversal time among the $k$ paths.
  For a graph whose list of edge traversal time distributions is $\boldsymbol{D}$, let $r_{\boldsymbol{D}}(\boldsymbol{p})$ be the expected value of the traversal time of the path with the longest traversal time for the observed traversal times $\{X_{i,j}\}$ of a set of $k$ paths $\boldsymbol{p}$, and let $\p^\star$ be the solution of the problem of minimizing this expected value. That is, we let
  \begin{align*}
    r_{\boldsymbol{D}}(\boldsymbol{p})=&\E_{\{X_{i,j}\sim D_i\}}\left[\max_{1\le j\le k}\sum_{i\in p_j}X_{i,j}\right]\\
    \p^\star=&\argmin_{\boldsymbol p\in\mathcal F} r_{\boldsymbol{D}}(\p)
  \end{align*}
  We define the regret gap of a set of $k$ paths $\p \in \F$ as
    \begin{align*}
      \Delta_\p = r_{\boldsymbol{D}} (\p) - r_{\boldsymbol{D}} (\p^{\star}) \geq 0
    \end{align*}
    We define the cumulative regret up to time $n$ of an algorithm $\pi$ by the following equation.
    \begin{align}
      \mathfrak R_{\boldsymbol{D}}^\pi(n)=&\E\Big[\sum_{t=1}^{n} L(t)\Big]-n\cdot r_{\boldsymbol{D}}(\boldsymbol p^\star)\nonumber\\
      =&\E\Big[\sum_{t=1}^{n}\big(r_{\boldsymbol{D}}(\p(t))-r_{\boldsymbol{D}}(\p^\star)\big)\Big]
      =\E\Big[\sum_{t=1}^{n}\Delta_{\p(t)}\Big].\label{eq:regret-identity}
    \end{align}

\section{Proposed Method}%
  As a method for solving this problem, we propose SDCB-MMLP (SDCB for MinMax-Length Path) (Algorithm \ref{alg:alg1}), an algorithm that uses the SDCB policy \cite{generalreward-Chen}.
  \begin{algorithm}[tb]
  \caption{SDCB-MMLP}
  \label{alg:alg1}
  \begin{algorithmic}[1]
  \State $U \gets E$, $n \gets 1$, $T_i(0)\gets 0$, $\hat{F}_i(x)\gets 0$
  \While {$U \ne \emptyset$}\label{alg:init-begin}
    \State $\boldsymbol{p}(n) \gets \argmax_{\boldsymbol{p}\in \mathcal{F}} |\{ i \in U\;|\; a_i(\boldsymbol{p}) > 0\}|$
  \State $\begin{array}[t]{l}{\{\hat{F}_{i,T_i(n)}(x)\},\{T_i(n)\}\gets}\\\!\!\text{Draw\&Update}(\{\hat{F}_{i,T_i(n-1)}(x)\},\{T_i(n-1)\},\boldsymbol{p}(n))\end{array}$
  \State $U \gets U \backslash \{ i \in E \;|\; a_i(\boldsymbol{p}) > 0\}, n \gets n + 1$
  \EndWhile\label{alg:init-end}
  \Loop
  \State $\boldsymbol{p}(n) \gets \argmin_{\boldsymbol{p} \in \mathcal{F}} r_{\bar{\boldsymbol{D}}(n)}(\p) $ \label{algline:oracle}
       \State $\begin{array}[t]{l}\!\!\!{\{\hat{F}_{i,T_i(n)}(x)\},\{T_i(n)\}\gets}\\\!\!\!\text{Draw\&Update}(\{\hat{F}_{i,T_i(n-1)}(x)\},\{T_i(n-1)\},\boldsymbol{p}(n))\end{array}$
    \State $n \gets n + 1$
  \EndLoop
  \end{algorithmic}
  \textbf{Procedure Draw\&Update}$(\{\hat{F}_i\},\{T_i\},\boldsymbol{p})$
  \begin{algorithmic}[1]
  \State Select the set $\boldsymbol{p}$ and observe the traversal time $X_{i,j}$ of each traversed edge.
  \For{$i\in E$}
  \State $\hat{F}_i\gets \frac{\hat{F}_i\cdot T_i+\sum_{j=1}^k\sum_{i\in p_j}\mathbbm{1}\{X_{i,j}\le x\}}{T_i+a_i(\boldsymbol{p})}$
  \State $T_i \gets T_i +a_i(\boldsymbol{p})$
  \EndFor
  \State \Return $\{\hat{F}_i\}, \{T_i\}$
  \end{algorithmic}
  \end{algorithm}
  Let $T_i(n)$ be the number of times each edge $i$ has been selected at the end of time $n$, and let $\hat{F}_{i,T_i(n)}$ be the CDF (cumulative distribution function) of the estimated distribution of the traversal time.
  As the CDF $\hat{F}_{i,T_i(n)} (x)$ at the end of time $n$, we use the following function:
  \begin{align*}
    \hat{F}_{i,T_i(n)} (x) = \frac{1}{T_i(n)} \sum_{t=1}^{n}\sum_{j=1}^k\sum_{i\in p_j(t)}\mathbbm{1}\{X_{i,j}(t) \leq x\}.
  \end{align*}
  For the estimated distribution whose cumulative distribution function is $\hat{F}_{i,T_i(n-1)}$, we define the cumulative distribution function of the optimistic distribution $\bar{D}_{i}(n)$ as
  \begin{align*}
    \overline{F}_{i,n} (x) = \begin{cases}
      \min \{\hat{F}_{i,T_i(n-1)} (x) + c_{i,n}, 1\} & (0 \leq x < 1) \\
      1 & (x=1).
    \end{cases}
  \end{align*}
  Here, $c_{i,n}$ denotes the width of the confidence interval, and is defined as
  \begin{align*}
    c_{i,n}=\sqrt{\frac{3\ln n}{2T_{i}(n-1)}}.
  \end{align*}
 
  In the objective function $r_{\bar{\boldsymbol{D}}(n)}(\p)$ of the minimization used in line \ref{algline:oracle} of the algorithm,
$\bar{\boldsymbol{D}}(n)=(\bar{D}_1(n),\cdots, \bar{D}_{|E|}(n))$
%  holds,
is used as a list of edge traversal time distributions,
  and the expected value is taken with respect to the optimistic distribution $\bar{D}_i(n)$ of each edge $i$.
  The number of selections of edge $i$ at time $0$ is set to $T_i(0)=0$, and the update of the number of selections at time $n$ can be performed by the equation
  \begin{align*}
    T_j (n+1) =  T_j (n) + a_j(\boldsymbol{p}(n))
  \end{align*}
  
\section{Upper Bound}
  Let $K$ be the number of edges of the longest path (the one with the maximum number of traversed edges) from vertex $1$ to $|V|$ in the directed graph $G$:
  \[
    K = \max_{\text{$p$ is a path from $1$ to $|V|$}} |p|,
  \]
  At each time $n$, the traversal times $L_j(n)$ and $L(n)$ of each path $p_j$ and of the longest path satisfy
  $ 0 \leq L_j(n) \leq K , \; 0 \leq L(n) \leq K$.
  Let $\bar{K}$ be the maximum total path length of a set $\p\in \F$ consisting of $k$ paths, and let $\bar{a}$ be the maximum number of occurrences of the same edge:
  \begin{align*}
    \bar{K}&=\max_{\p\in\mathcal{F}}\sum_{i\in E}a_i(\p)=\max_{\p\in\mathcal{F}}\sum_{j=1}^{k}|p_j|\;\leq\;kK,\\
    \bar{a}&=\max_{\p\in\mathcal{F}}\max_{i\in E}a_i(\p)\;\leq\;\min(k,\max_{i \in E} u_i)
  \end{align*}
  We define the family of suboptimal sets of paths $\mathcal{F}_{\mathrm{bad}}$ and the minimum gap $\Delta_{\min}$ between the suboptimal sets of paths and the optimal set of paths as follows:
  \begin{align*}
    \mathcal{F}_{\mathrm{bad}} = \{ \p \in \F | \Delta_{\p} > 0 \} ,\ \ 
    \Delta_{\min}=\min_{\p\in\mathcal{F}_{\mathrm{bad}}}\Delta_{\p}.
  \end{align*}
  Furthermore, we define the set of edges contained in suboptimal paths $E_{\mathrm{bad}}$ and the minimum gap $\Delta_{i,\min}$ between the suboptimal sets of paths containing edge $i$ and the optimal set of paths as follows:
  \begin{align*}
    E_{\mathrm{bad}}=&\{i\in E\mid \exists\,\p\in\mathcal{F}_{\mathrm{bad}},\,a_i(\p)>0\},\\
    \Delta_{i,\min}=&\min\{\Delta_{\p}\mid \p\in\mathcal{F}_{\mathrm{bad}},\,a_i(\p)>0\}.
  \end{align*}
  In the algorithm SDCB-MMLP, let $n_{\text{init}}$ be the number of iteration rounds of the first while loop (Line~\ref{alg:init-begin}-\ref{alg:init-end}), which is the initialization phase.
At time $t(>n_{\text{init}})$, we define the event that the empirical CDF of some edge %is inaccurate:
lies  outside the confidence interval:
\begin{align*}
    \mathcal{E}_t = \left\{ \exists i \in E : \sup_{x\in [0,1]} |\hat{F}_{i,T_{i}(t-1)}(x)- F_i(x)|\geq c_{i,t} \right\}.
  \end{align*}  
  We show the regret upper bound using the following Lemma~\ref{lem:nearcase} and Lemma~\ref{lem:dkw} [DKW (Dvoretzky-Kiefer-Wolfowitz) inequality].

\begin{lemma}\label{lem:nearcase}
  The algorithm SDCB-MMLP satisfies the following.
   \begin{enumerate}[label=(\arabic*)]
      \item (distribution-$\boldsymbol{D}$-dependent upper bound)
      \begin{align*}
       \sum_{t=n_{\text{init}}+1}^n \mathbbm{1}\{ \lnot \mathcal{E}_t \} \Delta_{p_t} 
        \leq&\sum_{i\in E_{\mathrm{bad}}}\frac{534\,\overline{a}^2\,\kappa}{\Delta_{i,\min}}\ln n
      \end{align*}
      \item (distribution-$\boldsymbol{D}$-independent upper bound)
      \begin{align*}
       &\sum_{t=n_{\text{init}}+1}^n \mathbbm{1}\{ \lnot \mathcal{E}_t \} \Delta_{p_t}\\ 
        \leq& (1+\sqrt{2})\sqrt{6|E|\bar{K}n\ln n} + \bar{a}|E|\sqrt{6\ln n}
      \end{align*}
    \end{enumerate}
\end{lemma}
\begin{proof}
    See Appendix~\ref{app:proof-nearcase}.
\end{proof}

\begin{lemma}[DKW inequality\cite{generalreward-Chen}]\label{lem:dkw}
For any $\epsilon>0$ and any positive integer $n$, the empirical CDF $\hat{F}_n(x)=\frac{1}{n}\sum_{i=1}^n\mathbbm{1}\{X_i\le x\}$ constructed from $n$ i.i.d. samples $X_1,\dots,X_n$ drawn from a distribution $D$ whose CDF is $F(x)$ satisfies the following inequality.
\[
\mathbbm{P}\left[\sup_{x\in\mathbbm{R}}|\hat{F}_n(x)-F(x)|\ge \epsilon\right]\le 2e^{-2n\epsilon^2}.
\]
\end{lemma}

\begin{theorem}\label{thm:main}
  The cumulative regret $\mathfrak{R}_{\boldsymbol{D}}^{\mathrm{SDCB-MMLP}}(n)$ of the algorithm SDCB-MMLP up to time $n$ satisfies the following.
    \begin{enumerate}[label=(\arabic*)]
      \item (distribution-$\boldsymbol{D}$-dependent upper bound)
      \begin{align*}
  &\mathfrak{R}_{\boldsymbol{D}}^{\mathrm{SDCB-MMLP}}(n)\\
        \leq&\sum_{i\in E_{\mathrm{bad}}}\frac{534\,\overline{a}^2\,\kappa}{\Delta_{i,\min}}\ln n+\Big(\tfrac{\pi^2}{3}\overline{a}+1\Big)K|E|\\
        =&O\!\Big(\frac{\overline{a}^2\,\kappa\,|E|}{\Delta_{\min}}\ln n\Big).
      \end{align*}
      \item (distribution-$\boldsymbol{D}$-independent upper bound)
      \begin{align*}
  &\mathfrak{R}_{\boldsymbol{D}}^{\mathrm{SDCB-MMLP}}(n)\\
        \leq& (1+\sqrt{2})\sqrt{6|E|\bar{K}n\ln n} \\ 
        &\qquad + \bar{a}|E|\sqrt{6\ln n} +\Big( \frac{\pi^2}{3}\bar{a} + 1\Big)K|E|\\
        =&O\!\big(\sqrt{|E|\,\bar{K}\,n\,\ln n}\big).
      \end{align*}
    \end{enumerate}
  \end{theorem}
\begin{proof}
From (\ref{eq:regret-identity}), we decompose the regret as follows.
    \begin{align*}
  & \mathfrak{R}^{{\mathrm{SDCB-MMLP}}}_n (D)\\
        =& \E \Bigg[ \sum_{t=1}^{n} \Delta_{\boldsymbol{p}(t)} \Bigg] \\
        =& \ub{\E \Bigg[ \sum_{t\leq n_{\text{init}}} \Delta_{\boldsymbol{p}(t)} \Bigg]}{\text{(A)}}
         + \ub{\E \Bigg[ \sum_{t=n_{\text{init}}+1}^n \mathbbm{1}\{ \mathcal{E}_t \} \Delta_{\boldsymbol{p}(t)} \Bigg]}{\text{(B)}} \\
        &+ \ub{\E \Bigg[ \sum_{t=n_{\text{init}}+1}^n \mathbbm{1}\{ \lnot \mathcal{E}_t \} \Delta_{\boldsymbol{p}(t)} \Bigg]}{\text{(C)}}
    \end{align*}
    \noindent
    (A) Regret %of
    during the initialization phase

    In the while loop of the initialization phase, the loop is continued until no unobserved edge remains.
    At this time, by the algorithm, at least one unobserved edge is observed in each iteration, and hence $n_{\text{init}} < |E|$.
    Since the maximum regret incurred in each round is bounded from above as $\Delta_{\p(t)}\leq r_D (\p(t)) \leq K $, we obtain
    \begin{align*}
      \text{(A)} \leq  \E \Bigg[ \sum_{t\leq n_{\text{init}}} \Delta_{\boldsymbol{p}(t)} \Bigg] \leq n_{\text{init}} K \leq |E|K
    \end{align*}
    \noindent
    (B) %The case where the empirical CDF deviates from the true CDF
    Regret while $\mathcal{E}_T$ occurs
    
    \begin{align*}
      &(B)=\E \Bigg[ \sum_{t=n_{\text{init}}+1}^n \mathbbm{1}\{ \mathcal{E}_t \} \Delta_{\boldsymbol{p}(t)} \Bigg] \\
      \leq& \sum_{t> n_{\text{init}}}\sum_{i=1}^{|E|}\E\left[\mathbbm{1}\{  \sup_{x} |\hat{F}_{i,T_i(t-1)}(x) - F_i (x)| \geq c_{i,t} \} \Delta_{\boldsymbol{p}(t)} \right]\\
      \leq& K\sum_{t> n_{\text{init}}} \sum_{i=1}^{|E|} \sum_{l=1}^{(t-1)\overline{a}}\\&\ \ \ \ \ \mathbbm{P}\Bigg[\sup_{x} |\hat{F}_{i,l}(x) - F_i (x)| \geq \sqrt{\frac{3\ln t}{2l}},T_i(t-1)=l \Bigg]\tag{by $\Delta_{\p(t)}\leq K$ and $T_i(t-1)\le(t-1)\bar{a}$}\\
      \leq& K\sum_{t> n_{\text{init}}} \sum_{i=1}^{|E|} \sum_{l=1}^{(t-1)\overline{a}} 2 t^{-3}\tag{by Lemma~\ref{lem:dkw}}\\
      \leq& 2\overline{a}|E|K \sum_{t > n_{\text{init}}} t^{-2} \\
      \leq& \frac{\pi^2}{3} \overline{a}|E|K\tag{by the solution of the Basel problem $\sum_{t=1}^{\infty}\frac{1}{t^2}=\frac{\pi^2}{6}$}
    \end{align*}
\noindent
    (C) %The case where the empirical CDF is close to the true CDF
Regret while $\mathcal{E}_T$ does not occur

  From Lemma~\ref{lem:nearcase}, the distribution-dependent and distribution-independent upper bounds are derived.
\end{proof}

\section{Experiments}

  Numerical experiments are conducted in order to verify the behavior of SDCB-MMLP.

\subsection{Experimental setting}
  \begin{figure}[t]
    \centering

    \begin{tikzpicture}[
      vertex/.style={circle, draw, fill=white, minimum size=8mm, inner sep=0pt, font=\small},
      arc/.style={thick, -{Stealth[length=2.5mm]}},
      x=1.6cm, y=1.6cm
    ]

    % ---- Placement of the vertices ----
    % Place vertex i at (column, row). column = i mod 5, row = floor(i/5)
    % By taking the y coordinate in the positive direction, row 0 is at the bottom and row 4 at the top
    % e.g., 0 -> (0,0) bottom left, 4 -> (4,0) bottom right, 20 -> (0,4) top left, 24 -> (4,4) top right
    \foreach \i in {0,...,24}{%
      \pgfmathtruncatemacro{\col}{mod(\i,5)}%
      \pgfmathtruncatemacro{\row}{int(\i/5)}%
      \node[vertex] (v\i) at (\col,\row) {\i};
    }

    % ---- Drawing of the directed edges (in the order of the given list, \a -> \b) ----
    \begin{pgfonlayer}{background}
    \foreach \a/\b in {%
      % Horizontal direction (within each row): from left to right
      0/1, 1/2, 2/3, 3/4,
      5/6, 6/7, 7/8, 8/9,
      10/11, 11/12, 12/13, 13/14,
      15/16, 16/17, 17/18, 18/19,
      20/21, 21/22, 22/23, 23/24,
      % Vertical direction (across rows): from bottom to top
      0/5, 1/6, 2/7, 3/8, 4/9,
      5/10, 6/11, 7/12, 8/13, 9/14,
      10/15, 11/16, 12/17, 13/18, 14/19,
      15/20, 16/21, 17/22, 18/23, 19/24%
    }{%
      \draw[arc] (v\a) -- (v\b);
    }
\end{pgfonlayer}

    \end{tikzpicture}
    \caption{The graph used in the experiment.}
    \label{fig:exp-graph}
  \end{figure}
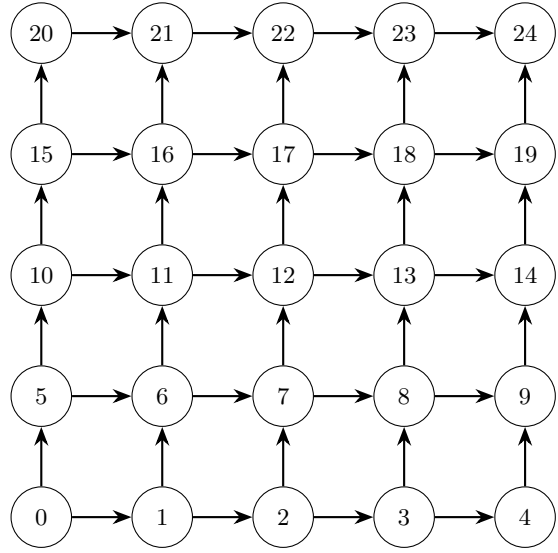

 %The experiment is performed on a graph such as the one in Figure \ref{fig:exp-graph} (number of vertices $|V|=25$, number of edges $|E|=40$).
In the experiment, we use the graph shown in Figure \ref{fig:exp-graph} (number of vertices $|V|=25$, number of edges $|E|=40$).
  Vertex 0 is the source and vertex 24 is the sink. The capacity $u_i$ of each edge $i$ is assigned according to the uniform distribution over $\{1,2,3,4\}$.
The mean $\theta_i$ of the traversal time distribution $D_i$ of each edge $i$ is generated according to the normal distribution $\mathcal{N}(0.5,0.1)$.
$D_i$ is a distribution over $\{0,0.2,0.4,0.6,0.8,1.0\}$, and the probability $\mathbbm{P}(X_i=a)$ of taking each value is the probability proportional to the probability density of the value $a$ under the normal distribution $\mathcal{N}(\theta_i, 0.15^2)$.
  
  In the experiment, the number of paths to be selected is set to $k=2$.
  From the shape of the graph and the problem setting, the existence of $2$ paths satisfying the capacity constraint is guaranteed.
  The proposed method SDCB-MMLP (Algorithm \ref{alg:alg1}) is compared with CLCB-MMLP, which is based on CUCB (CombUCB1)\cite{pmlr-v38-kveton15}.

  The experiment is run 20 times, and how the regret grows is observed up to a sufficiently large time ($t=250000$).
  %  Here, in order that the capacity constraints and the traversal times be the same for all the algorithms, the same random seed is given so that an identical graph is generated.
  For fair comparison, the same random seed for generating the capacity constraints and the traversal times is given for all the algorithms to make sure that an identical graph is %generated.
used.
  For random numbers and random sampling, the random number %generation engine implemented as mt19937 in the std library of C++ is used.
generator mt19937 in the std library of C++ is used.
\subsection{%Algorithm of the comparison method
Algorithm CLCB-MMLP
}
  
CLCB-MMLP, which is based on CUCB\cite{pmlr-v38-kveton15}, %minimizes the longest path by using lower confidence bounds.
solves the problem of selecting $k$ paths with the minimum longest path using the graph with the lower confidence bound (LCB) of the estimated traversal time for each edge.
%An LCB is assigned to each edge, and, using the number of selections $T_i(n-1)$ up to time $n-1$ and the observed mean $\hat{\theta}_i(n-1)$, $\text{LCB}_{n,i}$ for edge $i$ at time $n$ is defined as follows.
At each time $n$, an LCB for each edge $i$, denoted by $\text{LCB}_{n,i}$,  is calculated using the following equation:
    \begin{align*}
      \text{LCB}_{n,i} = 
      \hat{\theta}_i(n-1) - \sqrt{\frac{3\ln n}{2 T_i (n-1)}},
    \end{align*}
 where $\hat{\theta}_i(n-1)$ is the sample mean up to time $n-1$.
 %At each time $n$, assuming
Assuming that the traversal time of each edge $i$ is $\text{LCB}_{n,i}$,
CLCB-MMLP selects the set of $k$ paths %whose longest path traversal time
so that the longest traversal time of the component paths is the minimum among $k$ paths in $\mathcal{F}$. %is selected among the sets of $k$ paths from the sink to the source.
    
  \subsection{Experimental results}
    The experimental results are shown in Figure \ref{fig:exp-result}.

  \begin{figure}[t]
    \centering
    \includegraphics[width=8.0cm]{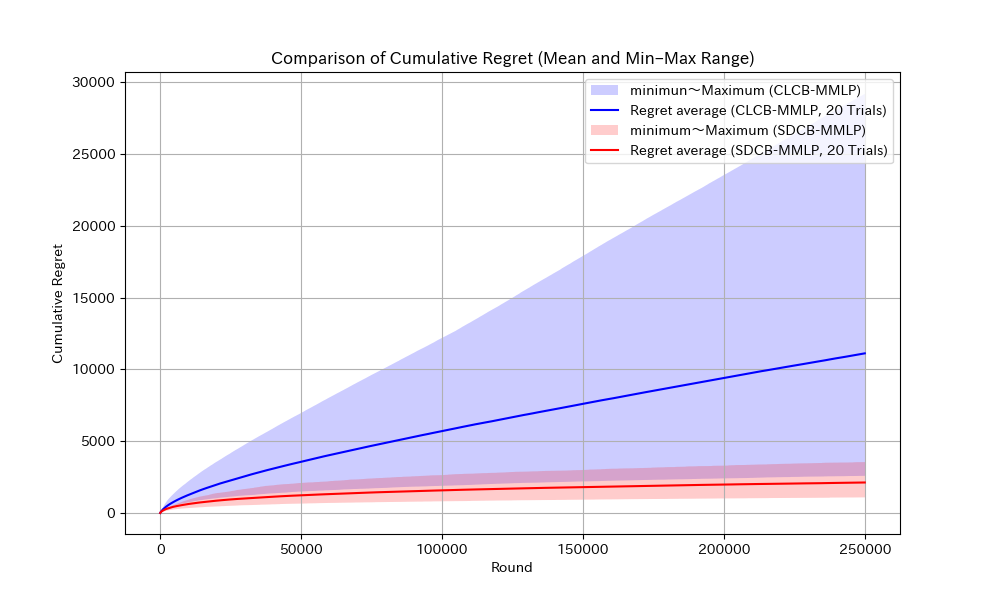}
    \caption{Experimental results}
    \label{fig:exp-result}
  \end{figure}

  %From the experimental results, for SDCB-MMLP, it was confirmed that the regret converges after a sufficient amount of time has elapsed in all the trials.
  %From this result, it was confirmed experimentally that the proposed method converges to the optimal solution as time elapses.
 From the results, we can see that the regret for SDCB-MMLP converges after a sufficient amount of time %has elapsed
  in all the trials, which indicates that the solution selected by SDCB-MMLP converges to the optimal solution.
 
  On the other hand, although the regret for CLCB-MMLP is similar to that of SDCB-MMLP in some trials, overall the regret did not converge and tended to increase roughly linearly.
  This is because the CUCB-based method uses estimators of expected rewards instead of estimators of the full distribution, and therefore there exist graphs for which the optimal set of paths cannot be learned.
  
  \subsection{%Limitations of the CUCB-based method
  Graph that CLCB-MMLP fails}

  In the $k$-path selection problem, the optimal solution depends not only on the mean but also on the shape of the reward distributions.
  For example, in the graph in Figure \ref{fig:kensyou-graph}, there are two parallel edges from vertex 1 to vertex 2: the upper edge has traversal time always equal to $1$ while the lower edge takes values $0$ or $2$ with equal probability. Their means are equal.
  %When two paths are selected, by the capacity
In the setting with $k=2$,
  the edge from vertex 1 to vertex 3 is necessarily traversed due to the capacity constraint, and the traversal time of this path is always 3.
  Therefore, among the two paths, the expected traversal time of the longest path is 3 when the upper edge is selected among the edges from vertex 1 to vertex 2, whereas it is 3.5 (4 with probability 1/2 and 3 with probability 1/2) when the lower edge is selected, and hence the optimal set of paths contains the upper edge.
  However, %the CUCB-based method, which computes only with the mean, cannot discern this difference.
  CLCB-MMLP cannot distinguish the optimal path set (paths $0\to 3$ and $0\to 1\xrightarrow{\text{upper edge}}2\to 3$) from the suboptimal path set
  (paths $0\to 3$ and $0\to 1\xrightarrow{\text{lower edge}}2\to 3$) because $\hat{\theta}_i(n-1)$ for both edges $i=(1\xrightarrow{\text{upper edge}}2)$ and $i=(1\xrightarrow{\text{lower edge}}2)$ converge to $1$ as $n\rightarrow \infty$.
\begin{figure}[tb]
  \centering
  \begin{tikzpicture}[
      scale=0.60, transform shape,
      >={Stealth[length=2.6mm,bend]},
      vtx/.style={circle,draw,thick,minimum size=8mm,inner sep=0pt},
      ed/.style ={->,thick},
      lbl/.style={font=\small,inner sep=2pt,fill=white},
    ]
    % ---- Vertices ----
    \node[vtx] (v0) at (0.0,0) {$0$};
    \node[vtx] (v1) at (3.2,0) {$1$};
    \node[vtx] (v2) at (6.4,0) {$2$};
    \node[vtx] (v3) at (9.6,0) {$3$};
    % ---- Edges ----
    \draw[ed] (v0) -- node[lbl,above]{$X\equiv 1,\ u=1$} (v1);
    \draw[ed] (v1) to[bend left=35]
      node[lbl,above]{$X_A\equiv 1,\ u=1$} (v2);
    \draw[ed] (v1) to[bend right=35]
      node[lbl,below]{$X_B\in\{0,2\}\ (\text{each }1/2),\ u=1$} (v2);
    \draw[ed] (v2) -- node[lbl,above]{$X\equiv 1,\ u=1$} (v3);
    \draw[ed] (v0) to[bend right=60]
      node[lbl,below]{$X\equiv 3,\ u=1$} (v3);   % <- ")" and ";" were supplied
  \end{tikzpicture}
  \caption{Graph %for verification (start vertex $0$, end vertex $3$).
    that CLCB-MMLP fails (its source and sink vertices are $0$ and $3$, respectively).
   % $X$ denotes the traversal time of an edge, and $u$ denotes the number of paths that can pass through.
    $X$ and $u$ denote the traversal time and capacity of an edge, respectively.
    $1\to 2$ is a multiple edge, and $X_A(\equiv 1)$ and $X_B$
    ($\Pr[X_B=0]=\Pr[X_B=2]=1/2$) have the same mean
    ($\mathbb{E}[X_A]=\mathbb{E}[X_B]=1$).}
  \label{fig:kensyou-graph}
\end{figure}
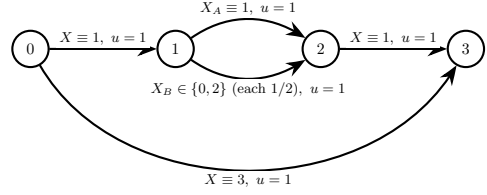
\iffalse
When we consider taking $k=2$ paths at every time in this case, the optimal solution $\p$ consists of the following two paths.
  \begin{itemize}
    \item $0 \to 3 $
    \item $0 \to 1 \to 2 \to 3$ (for $1\to 2$, $X_A$ is traversed.)
  \end{itemize}
  Moreover, whether or not the solution is optimal depends on which of the two edges the upper path of the graph traverses.
\fi
  
  \begin{table}[tb]
    \caption{Experimental results}
    \label{tab:exp2-result}
    \centering
  {\small
  \begin{tabular}{@{}rrrr@{}}
      \toprule
    \multirow{2}{*}{Algorithm}&  \multicolumn{2}{c}{\#selections of edge $1\rightarrow 2$} & \multirow{2}{*}{\begin{minipage}{1.5cm}Cumulative regret\end{minipage}}\\
      \cmidrule(lr){2-3}
                  & upper edge & lower edge & \\
      \midrule
      CLCB-MMLP & 2510 & 2490 & 1245 \\
      SDCB-MMLP & 4963 &   37 &  18.5 \\
      \bottomrule
  \end{tabular}
  }
  \end{table}
%  When CLCB-MMLP and SDCB-MMLP are actually run with $k=2$ paths taken at every time, the numbers of times each edge was selected up to time $t=5000$ were as shown in Table \ref{tab:exp2-result}.
%In Table \ref{tab:exp2-result}, $X_A$ and $X_B$ mean the numbers of selections of the deterministic edge of $1\to 2$ (whose traversal time is always 1) and of the stochastic edge (whose traversal time is $0$ or $2$ with equal probability), respectively.
%From the above results, it can be confirmed that there exist graphs that cannot be estimated by CUCB-type methods.
  Table~\ref{tab:exp2-result} shows the result of the experiment for our problem using the graph in Figure~\ref{fig:kensyou-graph} up to time $t=5000$.
From Table \ref{tab:exp2-result}, SDCB-MMLP selects $1\xrightarrow{\text{upper edge}}2$ almost all the time and achieves the optimal solution, whereas CLCB-MMLP selects $1\xrightarrow{\text{upper edge}}2$ and $1\xrightarrow{\text{lower edge}}2$ approximately equally and the solution did not converge.
This example confirms that there exist graphs that cannot be estimated by CUCB-type methods.

\section{Conclusion}
  In this paper we formulated an online version of the problem of selecting $k$ paths from the source to the sink in a directed graph so as to minimize the length of the longest selected path, and we proposed an algorithm to solve it.
  We derived an upper bound on the regret of the proposed algorithm and validated its effectiveness in numerical experiments.
  For grid graphs with randomly generated expected traversal times, the cumulative regret of the proposed method grows sublinearly and the method outperforms a CUCB-based baseline \cite{pmlr-v38-kveton15} in our experiments.
  
  For future work, it would be interesting to study regret lower bounds and to reduce the computational cost of the algorithm.
  A regret lower bound would indicate whether our upper bound is tight and how much room remains for improvement.
  The proposed method requires O(t) time to compute expectations from the CDF at time t because the empirical CDF is constructed from all past samples; this is more expensive than CUCB. Moreover, we used an exhaustive search to find the optimal set of paths, which is exponential-time.
  Since selecting $k$ paths that minimize the maximum path length is conjectured to be NP-hard, designing efficient approximation algorithms is desirable. Developing an online algorithm that employs an approximation oracle while retaining regret guarantees is another promising direction.

%%%%%%%%%%%%%%%%%%%%%%%%%%%%%%%%%%%%%%

\bibliographystyle{plain} %参考文献出力スタイル
\bibliography{cite} %hoge.bibから拡張子を外した名前

\appendix
\section{Proof of Lemma~\ref{lem:nearcase}}\label{app:proof-nearcase}

\subsection{Proof}

Let $S_t$ be the set of edges selected by the set of paths $\p(t)$, and let $\kappa$ be the upper bound on the number of edges traversed in one round:
\begin{align*}
      S_t =& \{ i \in E | a_i (\p(t)) > 0 \},\\
      \kappa=&\max_t |S_t|\leq \min(\bar{K}, |E|).
\end{align*}
The following corollary is necessary for the proof of Lemma~\ref{lem:g-happen}, which is used in the upper bound proof.
 \begin{corollary}\label{cor:onestep}
    Suppose that at time $t$, $|\hat{F}_{i,T_i(t-1)}(x)-F_i(x)|<c_{i,t}$ holds for all $i,x$. Then
    \begin{equation}\label{eq:onestep}
      \Delta_{\p(t)}\;\leq\; 2\sum_{i\in S_t}a_i(\p(t))\,c_{i,t}
      \;=\;\sqrt{6\ln t}\sum_{i\in S_t}\frac{a_i(\p(t))}{\sqrt{T_{i}(t-1)}}
    \end{equation}
holds.
  \end{corollary}
\begin{proof}
See Appendix~\ref{app:proof-onestep}.
\end{proof}

    By Theorem 4 of Kveton et al. (2015)\cite{pmlr-v38-kveton15}, there exist two positive decreasing sequences $1=\b_0 > \b_1 > \cdots , \a_1 > \a_2 > \cdots $ converging to 0 that satisfy the following condition (\ref{eq:abin6}):
    \begin{align}\label{eq:abin6}
      \sqrt{6}\sum_{u=1}^{\infty} \frac{\b_{u-1} - \b_u}{\sqrt{\a_u}} \leq 1,\ 
      \sum_{u=1}^{\infty} \frac{\a_u}{\b_u} < 267
    \end{align}
    
    For $n_{\text{init}}<t\le n$ and $u \in \mathbb{Z}_+$, where $n$ is an arbitrarily fixed stopping time, we define $m_{u,t}, A_{u,t}$ and $\mathcal{G}_{u,t}$ as follows:
    \begin{align*}
      &m_{u,t} = 
      \begin{cases}
          \a_u \Big( \frac{\bar{a} \; \kappa}{\Delta_{\p(t)}} \Big)^2 \ln n \quad &\Delta_{\p(t)} > 0 \\
          + \infty \quad &\Delta_{\p(t)} = 0, 
      \end{cases}\\
      &A_{u,t} = \{ i \in S_t | T_i(t-1)  \leq m_{u,t} \}, \ 
      \mathcal{G}_{u,t} = \{ |A_{u,t}| \geq \b_u \kappa \}
    \end{align*}
    Then, the following lemma holds.
    \begin{lemma}\label{lem:g-happen}
      If the event $\lnot\mathcal{E}_t$ occurs, then the event $\mathcal{G}_{u,t}$ also occurs for some $u \in \mathbb{Z}_+$.
    \end{lemma}
    \begin{proof}
     Assume that the event $\lnot\mathcal{E}_t$ occurs and that the event $\G_{u,t}$ occurs for no $u \in \mathbb{Z}_+$.
     Then $|A_{u,t}| < \b_u \kappa$ holds for all $u\in \mathbb{Z}_+$.
     Letting $\tilde{A}_{u,t} = S_t \backslash A_{u,t}$, we have $\tilde{A}_{1,t}\subseteq \tilde{A}_{2,t}\subseteq\cdots$, and $T_{i}(t-1) > m_{u,t} $ holds for every edge $i \in \tilde{A}_{u,t}$. Hence, in the same way as Lemma 5 of \cite{generalreward-Chen},
     \begin{align*}
      &\sum_{i \in S_t} \frac{1}{\sqrt{T_{i}(t-1)}} \\
      \leq & \sum_{i\in \tilde{A}_{1,t}}\frac{1}{\sqrt{m_{1,t}}}+\sum_{u=1}^{\infty}\sum_{i\in\tilde{A}_{u+1,t}\setminus \tilde{A}_{u,t}}\frac{1}{\sqrt{m_{u+1,t}}}\\
      =&\frac{|\tilde{A}_{1,t}|}{\sqrt{m_{1,t}}}+\sum_{u=1}^{\infty}\frac{|\tilde{A}_{u+1,t}|-|\tilde{A}_{u,t}|}{\sqrt{m_{u+1,t}}}\\
      =&\frac{|S_t|}{\sqrt{m_{1,t}}}+\sum_{u=1}^{\infty}\frac{|A_{u,t}|-|A_{u+1,t}|}{\sqrt{m_{u+1,t}}}\\
      \leq& 
          \frac{|S_t|}{\sqrt{m_{1,t}}} + \sum_{u=1}^{\infty} |A_{u,t}| \left( \frac{1}{\sqrt{m_{u+1, t}}} - \frac{1}{\sqrt{m_{u,t}}} \right)\\
      <& \frac{\kappa}{\sqrt{m_{1,t}}} + \sum_{u=1}^{\infty} \b_u \kappa \left( \frac{1}{\sqrt{m_{u+1,t}}} - \frac{1}{\sqrt{m_{u,t}}} \right)\\
      =& \sum_{u=1}^{\infty} \frac{\kappa(\b_{u-1}-\b_u)}{\sqrt{m_{u,t}}}
     \end{align*}
Since inequality (\ref{eq:onestep}) holds by Corollary~\ref{cor:onestep} under the assumption that the event $\lnot\mathcal{E}_t$ occurs, using $a_i(\p(t))\leq \bar{a}$, we obtain
      \begin{align*}
        \Delta_{\p(t)} 
        &\leq 2 \bar{a} \sum_{i \in S_t} c_{i,t} 
        = \bar{a} \sqrt{6\ln t} \sum_{i \in S_t} \frac{1}{\sqrt{T_{i}(t-1)}}\\
        &< \bar{a} \sqrt{6 \ln n} \sum_{u=1}^{\infty} (\b_{u-1} - \b_u) \frac{\Delta_{\p(t)}}{\bar{a} \sqrt{ \a_u \ln n}}\\
        &= \sqrt{6}\sum_{u=1}^{\infty} \frac{\b_{u-1} - \b_u}{\sqrt{\a_u}}\Delta_{\p(t)} 
        \leq \Delta_{\p(t)} 
      \end{align*} 
      The last inequality follows from (\ref{eq:abin6}).
      Thus the contradiction $\Delta_{\p(t)}<\Delta_{\p(t)}$ arises, and therefore it has been shown that if the event $\lnot\mathcal{E}_t$ occurs,
      then the event $\mathcal{G}_{u,t}$ occurs for some $u \in \mathbb{Z}_+$.
    \end{proof}

    \begin{proof}[\textbf{Proof of Lemma~\ref{lem:nearcase}}]
    The claim can be shown as follows, in the same way as the proof of Theorem 1 of \cite{generalreward-Chen}.
    By Lemma \ref{lem:g-happen},
    \begin{align}
       \mathbbm{1} \{ \lnot\mathcal{E}_t \}\leq \sum_{u=1}^{\infty}\mathbbm{1} \{ \mathcal{G}_{u,t} \}\label{ineq:Gut}
\end{align}    
    holds.
For an edge $i$, we define $\mathcal{G}_{i,u,t}$ as
\[
 \mathcal{G}_{i,u,t} = \mathcal{G}_{u,t} \land \{ i \in A_{u,t} \}
\]
Then, when $\Delta_{\p(t)} > 0$, we have $A_{u,t}\subseteq S_t \subseteq E_{\text{bad}}$, and
    on $\mathcal{G}_{u,t}$, $|A_{u,t}| \geq \b_u \kappa$ holds, so

    \begin{align}
      \mathbbm{1}\{ \mathcal{G}_{u,t}, \; \Delta_{\p(t)} > 0 \} \leq \frac{1}{\b_u \kappa} \sum_{i \in E_{\text{bad}}} \mathbbm{1} \{ \mathcal{G}_{i,u,t} , \Delta_{\p(t)} > 0 \}\label{ineq:Giut}
    \end{align}
holds.
%In a round in which arm $i$ is selected, $T_{i}(t-1)$ increases by at least 1, so the number of rounds with $T_{i}(t-1)\leq \tau $ is at most $\tau$.

    Let $N_i$ be the number of suboptimal sets of paths containing each $i\in E_{\bad}$, and let $P_{i,1},\dots P_{i,N_i}$ be these $N_i$ suboptimal sets of paths.
Here, the order of the sets of paths $P_{i,l}$ is determined so that the gaps $\Delta_{P_{i,l}}$ from the optimal set of paths are in descending order    
$\Delta_{P_{i,1}} \geq \dots \geq \Delta_{P_{i, N_i}} = \Delta_{i, \min}$, and for conveniencei, in the proof we let $\Delta_{P_i,0} = + \infty$.
Then
\begin{align}
   \mathbbm{1} \{ \mathcal{G}_{i,u,t} , \Delta_{\p(t)} > 0 \}\leq \sum_{l=1}^{N_i}\mathbbm{1} \{ \mathcal{G}_{i,u,t} , \p(t)=P_{i,l} \}\label{ineq:giutpt}
\end{align}
holds. Hence, from inequalities (\ref{ineq:Gut}), (\ref{ineq:Giut}) and (\ref{ineq:giutpt}),
%    Then, $\a_k (\bar{a}\kappa / \Delta_{i,0})^2 = 0$.
%    When $\p(t) = P_{i,l}^\bad$, we have $\Delta_{\p(t)} = \Delta_{i,l}$ and $m_{u,t} = \a_k (\bar{a}\kappa / \Delta_{i,l})^2 \ln n$.
%    In what follows, using Lemma \ref{lem:g-happen},
    \begin{align*}
      &\sum_{t=n_{\text{init}}+1}^n \mathbbm{1} \{ \lnot\mathcal{E}_t \}\Delta_{\p(t)}\\ 
      &\leq \sum_{i \in E_\bad} \sum_{u=1}^\infty \sum_{t=n_{\text{init}}+1}^n \sum_{l=1}^{N_i}
       \mathbbm{1} \{ \mathcal{G}_{i,u,t}, \p(t) = P_{i,l} \} \frac{\Delta_{P_{i,l}}}{\b_u \kappa}\\
      &= \sum_{i \in E_\bad} \sum_{u=1}^\infty \sum_{t=n_{\text{init}}+1}^n \sum_{l=1}^{N_i}\\ 
        &\ \ \ \ \ \ \mathbbm{1} \left\{ T_{i}(t-1)\leq \a_u \Big(\frac{\bar{a}\kappa}{\Delta_{P_{i,l}}}\Big)^2 \ln n, 
         \p(t) = P_{i,l} \right\} \frac{\Delta_{P_{i,l}}}{\b_u \kappa}\\
      &=\sum_{i \in E_\bad} \sum_{u=1}^\infty \sum_{t=n_{\text{init}}+1}^n \sum_{l=1}^{N_i} \sum_{j=1}^{l}\\ 
        &\mathbbm{1} \Bigg\{ \a_u \Big( \frac{\bar{a}\kappa}{\Delta_{P_{i,j-1}}} \Big)^2 \ln n 
        \leq T_{i}(t-1) \leq \a_u \Big( \frac{\bar{a}\kappa}{\Delta_{P_{i,j}}} \Big)^2\ln n,\\
        &\hspace{5.5cm}\p(t) = P_{i,l} \bigg\} \frac{\Delta_{P_{i,l}}}{\b_u \kappa}\\
      &\leq \sum_{i \in E_\bad} \sum_{u=1}^\infty \sum_{t=n_{\text{init}}+1}^n \sum_{l=1}^{N_i} \sum_{j=1}^{l}\\
       &\mathbbm{1} \Bigg\{ \a_u \Big( \frac{\bar{a}\kappa}{\Delta_{P_{i,j-1}}} \Big)^2 \ln n 
        \leq T_{i}(t-1) \leq \a_u \Big( \frac{\bar{a}\kappa}{\Delta_{P_{i,j}}} \Big)^2\ln n,\\
        &\hspace{5.5cm}\p(t) = P_{i,l} \bigg\} \frac{\Delta_{P_{i,j}}}{\b_u \kappa}\\
      &\leq \sum_{i \in E_\bad} \sum_{u=1}^\infty \sum_{t=n_{\text{init}}+1}^n \sum_{l=1}^{N_i} \sum_{j=1}^{N_i}\\
  &\mathbbm{1} \Bigg\{ \a_u \Big( \frac{\bar{a}\kappa}{\Delta_{P_{i,j-1}}} \Big)^2 \ln n 
        \leq T_{i}(t-1) \leq \a_u \Big( \frac{\bar{a}\kappa}{\Delta_{P_{i,j}}} \Big)^2\ln n,\\
        &\hspace{5.5cm}\p(t) = P_{i,l} \bigg\} \frac{\Delta_{P_{i,j}}}{\b_u \kappa}\\    
      &\leq \sum_{i \in E_\bad} \sum_{u=1}^\infty \sum_{t=n_{\text{init}}+1}^n \sum_{j=1}^{N_i}\\
  &\mathbbm{1} \Bigg\{ \a_u \Big( \frac{\bar{a}\kappa}{\Delta_{P_{i,j-1}}} \Big)^2 \ln n 
        \leq T_{i}(t-1) \leq \a_u \Big( \frac{\bar{a}\kappa}{\Delta_{P_{i,j}}} \Big)^2\ln n,\\
        &\hspace{5.5cm}i\in S_t \bigg\} \frac{\Delta_{P_{i,j}}}{\b_u \kappa}\\        
      &\leq \sum_{i \in E_\bad} \sum_{u=1}^\infty \sum_{j=1}^{N_i}\\ 
        &\ \ \ \ \ \Biggl( \a_u \Big( \frac{\bar{a}\kappa}{\Delta_{P_{i,j-1}}} \Big)^2 \ln n - \a_u \Big( \frac{\bar{a}\kappa}{\Delta_{P_{i,j}}} \Big)^2 \ln n \Biggr) \frac{\Delta_{P_{i,j}}}{\b_u \kappa}\\
      &= \bar{a}^2 \kappa \Biggl( \sum_{u=1}^{\infty} \frac{\a_u}{\b_u} \Biggr) \ln n \sum_{i \in E_{\bad}}\sum_{j=1}^{N_i}\Biggl( \frac{1}{\Delta_{P_{i,j}}^2} - \frac{1}{\Delta_{P_{i,j-1}}^2} \Biggr)\Delta_{P_{i,j}} \\
      &\leq 267 \bar{a}^2 \kappa \ln n \sum_{i \in E_{\bad}}\sum_{j=1}^{N_i}\Biggl( \frac{1}{\Delta_{P_{i,j}}^2} - \frac{1}{\Delta_{P_{i,j-1}}^2} \Biggr)\Delta_{P_{i,j}}, 
    \end{align*}
    where the last inequality holds by Eq.~(\ref{eq:abin6}).
    For each $i\in E_{\bad}$,
    \begin{align*}
      &\sum_{j=1}^{N_i} \Biggl( \frac{1}{\Delta_{P_{i,j}}^2} - \frac{1}{\Delta_{P_{i,j-1}}^2 } \Biggr)\Delta_{P_{i,j}}\\ =&\frac{1}{\Delta_{P_{i,N_i}}}+\sum_{j=1}^{N_i-1}\frac{1}{\Delta^2_{P_{i,j}}}(\Delta_{P_{i,j}}-\Delta_{P_{i,j+1}})\\
      \le& \frac{1}{\Delta_{P_{i,N_i}}}+\int_{\Delta_{P_{i,N_i}}}^{\Delta_{P_{i,1}}}\frac{1}{x^2}dx\\
      =& \frac{2}{\Delta_{P_{i,N_i}}}-\frac{1}{\Delta_{P_{i,1}}}< \frac{2}{\Delta_{i,\min}}
    \end{align*}
    From the above,
    \begin{align*}
      \sum_{t> n_{\text{init}}}^T \mathbbm{1} \{ \lnot\mathcal{E}_t \}\Delta_{\p(t)} \leq 267 \bar{a}^2 \kappa \ln n \sum_{i \in E_{\bad}} \frac{2}{\Delta_{i,\min}} \\
      = \bar{a}^2 \kappa \sum_{i \in E_\bad} \frac{534}{\Delta_{i,\min}} \ln n 
    \end{align*}
    
    Next, we directly evaluate the distribution-independent upper bound,  
    using the fact that the increment $a_i(\p(t))$ of the number of observations of each edge $i$ in one round is equal to $T_{i}(t) - T_{i}(t-1)$.
    From equation (\ref{eq:onestep}) of Corollary \ref{cor:onestep}, under the event $\lnot\mathcal{E}_t$,
    \begin{align*}
      \sum_{t=n_{\text{init}}+1}^n \mathbbm{1} \{\lnot\mathcal{E}_t \}\Delta_{\p(t)}
      \leq \sqrt{6 \ln n} \sum_{i \in E}\sum_{t = n_{\text{init}}+1}^n \frac{a_i(\p(t))}{\sqrt{T_{i}(t-1)}}
    \end{align*}

    %For each edge $i$, when it has been observed at least $T_{i}(t-1) \geq \bar{a}$ times, the number of observations $a_i(\p(t))$ at a given time is bounded by the maximum value $\bar{a}$ that it can take, and hence
For any edge $i$ with $T_{i}(t-1) \geq \bar{a}$, 
    \begin{align*}
      T_{i}(t) 
      &=    T_{i}(t-1) + a_i(\p(t)) \\
      &\leq T_{i}(t-1) + \bar{a} \\
      &\leq T_{i}(t-1) + T_{i}(t-1) = 2T_{i}(t-1)
    \end{align*}
    holds. Therefore, in a round in which $T_{i}(t) \leq 2T_{i}(t-1)$ holds,
    \begin{align*}
      \frac{a_i(\p(t))}{\sqrt{T_{i}(t-1)}} &= \frac{T_{i}(t)-T_{i}(t-1)}{\sqrt{T_{i}(t-1)}}\\
      &= \frac{\sqrt{T_{i}(t)} + \sqrt{T_{i}(t-1)}}{\sqrt{T_{i}(t-1)}}(\sqrt{T_{i}(t)} -\sqrt{T_{i}(t-1)}) \\
      &\leq (1 + \sqrt{2})( \sqrt{T_{i}(t)} -\sqrt{T_{i}(t-1)})
    \end{align*}
    holds.
    Hence,
    \begin{align*}
      \sum_{t=n_{\text{init}}+1}^n \mathbbm{1} \{\lnot\mathcal{E}_t \}\Delta_{\p(t)}
      &\leq \sqrt{6 \ln n} \sum_{i \in E}\sum_{t = n_{\text{init}}+1}^n \frac{a_i(\p(t))}{\sqrt{T_{i}(t-1)}}\\
      &\leq\sqrt{6 \ln n} \sum_{i \in E} ((1+\sqrt{2})\sqrt{T_{i}(n)}+\bar{a})
    \end{align*}
    holds.
    By the Cauchy--Schwarz inequality and
    $\sum_i T_{i}(n)= \sum_{t=1}^{n}\sum_{i\in E} a_i(\p(t)) \leq \bar{K}n$,
    \begin{align*}
      \sum_{i\in E} \sqrt{T_{i}(n)} 
      \leq \sqrt{|E|\sum_{i\in E}T_{i}(n)}
      \leq \sqrt{|E|\bar{K}n}
    \end{align*}
    From the above, we obtain the following distribution-independent upper bound.
    \begin{align*}
      &\sum_{t=n_{\text{init}}+1}^n \mathbbm{1} \{\lnot\mathcal{E}_t \}\Delta_{\p(t)}\\
      \leq& (1+\sqrt{2})\sqrt{6\ln n}\sqrt{|E|\bar{K}n} + \bar{a}|E|\sqrt{6\ln n}\\
      \leq& (1+\sqrt{2})\sqrt{6|E|\bar{K}n\ln n} + \bar{a}|E|\sqrt{6\ln n}
    \end{align*}
  \end{proof}

\subsection{Proof of Corollary~\ref{cor:onestep}}\label{app:proof-onestep}

  \begin{lemma}\label{lem:part}
    For a set of paths $\p=\{p_1,\dots,p_k\}$, let $L_j$ be the traversal time of the path $p_j$. Then, letting $L_{\bar{j}}= \max_{j' \ne j} L_{j'}$ denote the path length of the path $p_{j'}$ ($j'\neq j$) whose traversal time is the largest among the other paths,
    \begin{align*}
      &\E[\max(L_{\bar{j}},L_j)]= \E[L_{\bar{j}}]+\E[(L_j-L_{\bar{j}})_+]\\
      &= \E[L_{\bar{j}}]+\int_{L_{\bar{j}}}^{K}\big(1-F_{L_j}(\ell)\big)\,d\ell
    \end{align*}
    where $(L_j - L_{\bar{j}})_+ = \max(L_j - L_{\bar{j}} , 0)$ and $F_{L_j}$ is the CDF of the random variable $L_j$.
  \end{lemma}

  \begin{proof}
      Since $\max(L_{\bar{j}},L_j) = L_{\bar{j}} + (L_j - L_{\bar{j}})_+$ holds, %by the linearity of the expectation operation,
\[
\E[\max(L_{\bar{j}},L_j)]= \E[L_{\bar{j}}]+\E[(L_j-L_{\bar{j}})_+]
\]
holds by the linearity of the expectation operation.
    We transform the second term on the right-hand side of this equation into an integral form.
    Now, for any nonnegative real number $Z \geq 0 $, the following integral identity holds:
    \begin{align*}
      Z=\int_{0}^{\infty} \mathbbm{1} \{  Z > t \} dt.
    \end{align*}
 %   Setting $Z= (L_j - L_{\bar{j}})_+$, since $t\geq 0$,
Since 
    \begin{align*}
      (L_j - L_{\bar{j}})_+ > t \iff L_j - L_{\bar{j}} > t
    \end{align*}
  holds for $t\geq 0$,
    \begin{align*}
      (L_j - L_{\bar{j}})_+ = \int_{0}^{\infty} \mathbbm{1} \{ L_j - L_{\bar{j}} > t \}dt
    \end{align*}
    holds by setting $Z= (L_j - L_{\bar{j}})_+$.
   % Substituting the variable $l=L_{\bar{j}}+t$, %$dt=dl$,
 Substituting a new variable $\ell$ for $L_{\bar{j}}+t$, equation $dt=d\ell$ holds    
    and the range of integration $t:0\to \infty$ changes to $\ell:L_{\bar{j}} \to \infty$, thus we obtain
    \begin{align*}
      (L_j - L_{\bar{j}})_+ =& \int_{L_{\bar{j}}}^{\infty} \mathbbm{1} \{ L_j > \ell \}d\ell.
    \end{align*}
Using the fact that $L_j \leq K$,
%    Here, using the fact that the paths are bounded, we have $L \leq K$. Hence, for a variable $l$ with $l\geq K$, the condition $L>l$ inside the indicator variable does not hold, so
%    the integrand becomes 0. Therefore, the range of integration can be bounded from above, and
    \begin{align*}
      (L_j - L_{\bar{j}})_+ = \int_{L_{\bar{j}}}^{K} \mathbbm{1} \{ L_j > \ell \}d\ell
    \end{align*}
holds.
    Taking the expectation of both sides, we have
    \begin{align*}
      \E_{L_j}[(L_j - L_{\bar{j}})_+] = \E_{L_j} \Biggl[ \int_{L_{\bar{j}}}^{K} \mathbbm{1} \{ L_j > \ell \}d\ell \Biggr].
    \end{align*}
    Switching the order of integration by Fubini--Tonelli theorem, %the probability measure and the integral can be interchanged, and hence
    \begin{align*}
      \E_{L_j}[(L_j - L_{\bar{j}})_+]  =& \int_{L_{\bar{j}}}^{K}  \E_{L_j} \Biggl[\mathbbm{1} \{ L_j > \ell \} \Biggr]d\ell\\
      =& \int_{L_{\bar{j}}}^{K} \mathbbm{P}[L_j > \ell] d\ell 
    \end{align*}
    holds.
    From $\mathbbm{P}[L_j > \ell] = 1 - \mathbbm{P}[L_j \leq \ell] = 1 - F_{L_j}(\ell)$,
    \begin{align*}
     \E_{L_j}[(L_j - L_{\bar{j}})_+]  = 
      \int_{L_{\bar{j}}}^{K} (1- F_{L_j}(\ell))d\ell
    \end{align*}
    holds. Therefore, the following holds.
        \begin{align*}
      \E[\max(L_{\bar{j}},L_j)]
      = \E[L_{\bar{j}}]+\int_{L_{\bar{j}}}^{K}\big(1-F_{L_j}(\ell)\big)\,d\ell
    \end{align*}
  \end{proof}

 \begin{lemma}\label{lem:dominance}
    For each $i\in E$, let $\overline{D}_i$, $D_i$ be distributions on $[0,1]$, and let their CDFs be $\overline{F}_i$, $F_i$.
    \begin{enumerate}[label=(\roman*)]
      \item If $\overline{F}_i(x)\geq F_i(x)$ for all $i\in E$ and $x\in[0,1]$, then for any $\p\in\mathcal{F}$,
      \[
        r_{\overline{\boldsymbol{D}}}(\p)\leq r_{\boldsymbol{D}}(\p).
      \]
      \item If, furthermore, $\overline{F}_i(x)-F_i(x)\leq \Lambda_i$($\Lambda_i\geq 0$) for all $i\in E$ and $x\in [0,1]$, then for any $\p\in\mathcal{F}$,
      \begin{align*}
        0\;\leq\; r_{\boldsymbol{D}}(\p)-r_{\overline{\boldsymbol{D}}}(\p)\;
        &\leq\sum_{i\in E}a_i(\p)\,\Lambda_i.
      \end{align*}
    \end{enumerate}
  \end{lemma}
  \begin{proof}
    (i) The assumption $\overline{F}_i\geq F_i$ means that $\overline{D}_i$ is first-order stochastically dominated by $D_i$ (the probability mass is shifted toward smaller values). Since $r_D(\p)$ is nondecreasing in each component, replacing the distribution $D_i$ by $\overline{D}_i$ does not increase the expectation. Hence $r_{\overline{\boldsymbol{D}}}(\p)\leq r_{\boldsymbol{D}}(\p)$.

    (ii) Using the hybrid argument, which is commonly used in cryptography, we give the following proof by replacing the paths one by one from the true distribution to the optimistic distribution.
    
   For a path $p_u$ ($1\le u\le k$), suppose that the observed traversal time $X_{i,u}$ for an edge $i\in p_u$ is determined as
    \begin{align*}
      X_{i,u} \sim 
      \begin{cases}
        \D_i & (u \leq j) \\  
        D_i  & (u > j)   
      \end{cases}
    \end{align*}
    and let $H_j$ ($0\le j\le k$) be the joint distribution of $\{X_{i,u}\}$ determined in this way. Under the distribution $H_j$,
    the observed values for each path are independent, and the values from an edge traversed by multiple paths are also mutually independent.
    Thus,
    \begin{align*}
    &  r_{\boldsymbol{D}}(\p)-r_{\overline{\boldsymbol{D}}}(\p)\\=&\E_{H_0}[\max_j L_j]-\E_{H_k}[\max_j L_j]\\
                                  =&\sum_{m=1}^{k}\Big(\E_{H_{m-1}}[\max_j L_j]-\E_{H_{m}}[\max_j L_j]\Big).
    \end{align*}
    holds.
    In the $m$-th term, only the distribution of path $m$ changes from $\boldsymbol{D}$ to $\overline{\boldsymbol{D}}$, while the other paths are fixed and independent.
    Hence, letting $L_{\bar{m}}=\max_{j'\neq m}L_{j'}$ be the maximum among the other paths,
    $L_{\bar{m}}$ is generated from the same distribution under both $H_{m-1}$ and $H_m$, and is independent of $L_m$. Therefore, for a fixed $L_{\bar{m}}$, by Lemma \ref{lem:part},
    \begin{align*}
   &\E[\max(L_{\bar{m}},L_m)]=
      \E[L_{\bar{m}}]+\int_{L_{\bar{m}}}^{K}\big(1-F_{L_m}(\ell)\big)\,d\ell
    \end{align*}
    holds. Thus, letting $L_m^{\boldsymbol{D}}$, $L_m^{\overline{\boldsymbol{D}}}$ be the total path lengths of the path $p_m$ drawn from the distributions $\boldsymbol{D}$, $\overline{\boldsymbol{D}}$, respectively,
    \begin{align*} 
      &\E[\max(L_{\bar{m}},L^{\boldsymbol{D}}_m)]-\E[\max(L_{\bar{m}},L^{\overline{\boldsymbol{D}}}_m)]\\
      =&\int_{L_{\bar{m}}}^{K}\big(\overline{F}_{L_m}(\ell)-F_{L_m}(\ell)\big)\,d\ell\;\geq\;0.
    \end{align*}
    The last inequality follows from $\overline{F}_{L_m}\geq F_{L_m}$, which holds for the same reason as in (i).
    Taking the expectation with respect to $L_{\bar{m}}$,
    \begin{align*}  
      &\E_{H_{m-1}}[\max_j L_j]-\E_{H_{m}}[\max_j L_j]\\
      &= \E_{L_{\bar{m}}} \bigg[ \int_{L_{\bar{m}}}^{K}\big(\overline{F}_{L_m}(\ell)-F_{L_m}(\ell)\big) d\ell \bigg]\\
      &= \E_{L_{\bar{m}}}  \bigg[ \int_0^{K} \mathbbm{1}\{ L_{\bar{m}} \leq \ell \} \big(\overline{F}_{L_m}(\ell)-F_{L_m}(\ell)\big) d\ell \bigg] \qquad \\%(the range of integration written with an indicator variable) \\
      &= \int_0^{K} \E_{L_{\bar{m}}} [\mathbbm{1}\{ L_{\bar{m}} \leq \ell \}] \big(\overline{F}_{L_m}(\ell)-F_{L_m}(\ell)\big) d\ell \qquad \\% ( by Tonelli's theorem ) \\
    &\le \int_0^{K} \big(\overline{F}_{L_m}(\ell)-F_{L_m}(\ell)\big) d\ell
    \end{align*}
    is obtained. Therefore,
    \begin{align*}
      r_{\boldsymbol{D}}(\p)-r_{\overline{\boldsymbol{D}}}(\p)&\leq\sum_{j=1}^{k}\int_0^{K}\big(\overline{F}_{L_j}(\ell)-F_{L_j}(\ell)\big)\,d\ell \\
&=\sum_{j=1}^{k}\sum_{i\in p_j}\int_0^{1}\big(\overline{F}_{i}(\ell)-F_{i}(\ell)\big)\,d\ell\\
&\le \sum_{j=1}^{k}\sum_{i\in p_j}\Lambda_i=\sum_{i\in E}a_i(\p)\Lambda_i
    \end{align*}
    Therefore, the inequality in (ii) follows.    
  \end{proof}

  \begin{proof}[\textbf{Proof of Corollary~\ref{cor:onestep}}]
    By assumption $F_i(x)<\hat{F}_{i,T_i(t-1)}(x)+c_{i,t}$ and fact $F_i(x)\leq 1$,
    \[\overline{F}_{i,t}(x)=\min\{\hat{F}_{i,T_i(t-1)}(x)+c_{i,t},1\}\geq F_i(x) \] holds.
    Also, from $\hat{F}_{i,T_i(t-1)}(x)<F_i(x)+c_{i,t}$, we have $\overline{F}_{i,t}(x)\leq F_i(x)+2c_{i,t}$.
    Hence, with $\Lambda_i=2c_{i,t}$, Lemma \ref{lem:dominance}(i)(ii) can be applied, and combining this with %the optimality
    $r_{\overline{\boldsymbol{D}}}(\p(t))\leq r_{\overline{\boldsymbol{D}}}(\p^\star)$ ($\p(t)$'s optimality),
    \[
      r_{\boldsymbol{D}}(\p(t))-2\sum_i a_i(\p(t))c_{i,t}\leq r_{\overline{\boldsymbol{D}}}(\p(t))\leq r_{\overline{\boldsymbol{D}}}(\p^\star)\leq r_{\boldsymbol{D}}(\p^\star),
    \]
    that is, $\Delta_{\p(t)}=r_{\boldsymbol{D}}(\p(t))-r_{\boldsymbol{D}}(\p^\star)\leq 2\sum_i a_i(\p(t))c_{i,t}$. The transformation of the equality is due to $2c_{i,t}=2\sqrt{3\ln t/(2T_{i}(t-1))}=\sqrt{6\ln t}/\sqrt{T_{i}(t-1)}$.
  \end{proof}

\end{document}